\documentclass[11pt]{article}
\usepackage[utf8]{inputenc}
\pdfoutput=1
\usepackage{setspace}
\usepackage{palatino}
\usepackage{graphicx}
\usepackage{float}
\usepackage{titling} 
\usepackage{multirow}
\usepackage{lscape}
\usepackage{amsmath}
\usepackage{amssymb}
\usepackage[a4paper, total={6in, 9.5in}]{geometry}
\usepackage{array}
\usepackage[normalem]{ulem}
\fontfamily{ppl}\selectfont 
\usepackage{eqparbox}
\usepackage{arydshln}
\usepackage{mathrsfs}
\usepackage[american]{babel}
\usepackage{mathtools}

\usepackage{rotating}
\usepackage{amsthm}
\newtheorem{theorem}{Theorem}

\usepackage[ruled,vlined]{algorithm2e}

\usepackage{bm}
\usepackage{caption}
\usepackage{subcaption}

\usepackage[natbibapa]{apacite}
\PassOptionsToPackage{hyperindex,breaklinks}{hyperref}
\usepackage{hyperref}
\usepackage{xcolor}
\definecolor{darkblue}{rgb}{0, 0, 0.5}
\hypersetup{colorlinks=true,citecolor=darkblue, linkcolor=darkblue, urlcolor=darkblue}

\usepackage[textsize=small]{todonotes}

\AtBeginDocument{%
  \renewcommand{\doi}[1]{%
    \href{https://doi.org/#1}{\nolinkurl{https://doi.org/#1}}%
  }%
}

\title{Maximum-distance nonnegative matrix factorization for unmixing highly mixed grain-size distribution data}

\author{Qianqian Qi\\
   {\small\raggedright Hangzhou Dianzi University, China}\\
   \href{mailto:q.qi@hdu.edu.cn}{\texttt{q.qi@hdu.edu.cn}} 
\and Zhongming Chen\\
    {\small\raggedright Hangzhou Dianzi University, China}\\
\href{mailto:zmchen@hdu.edu.cn}{\texttt{zmchen@hdu.edu.cn}}
\and Peter G. M. van der Heijden\\
    {\small\raggedright Utrecht University, the Netherlands and University of Southampton, UK}\\
\href{mailto:p.g.m.vanderheijden@uu.nl}{\texttt{p.g.m.vanderheijden@uu.nl}}
    }
    
\predate{}
\postdate{}
\date{}
\date{\vspace{-5ex}}

\begin{document}

{\setstretch{.8}
\maketitle
\begin{abstract}

Unmixing grain-size distribution data into several end-members can help reveal sediment provenance, transport, and sedimentation processes. Traditional end-member unmixing methods seek end-members closely to specimens, which is suitable when data is poorly mixed. However, these methods can be limited for highly mixed data when the observed specimens are not near to the underlying end-members. To overcome this limitation, we introduce a maximum-distance nonnegative matrix factorization model that encourages the estimated end-members to be as distinct as possible. We adopt a hierarchical alternating least squares algorithm for optimization and analyze the convergence properties. We evaluate its performance on simulated grain-size distribution data and compare it with existing method. The results illustrate the potential of the proposed approach for end-member unmixing in highly mixed data.
\end{abstract}
}

\section{Introduction}\label{S: Introduction}

Sediment grain-size distribution data often exhibit multiple peaks, reflecting different transport modes and sedimentation processes \citep{weltje1997end, van2018genetically, ZHANG2020106656}. Unmixing grain-size distribution data into several end-members with geological significance can help in revealing sediment provenance, transport, sedimentation processes, et. al. \citep{renner1993resolution, renner1995construction, weltje1997end, weltje2007genetically, paterson2015new, van2018genetically, Dietze2022application, Lin2025Using, qi2026unmixing}. Let $\bm{P} \in \mathbb{R}^{m\times n}_+$ denote a grain-size distribution data, where each row represents a specimen and each column corresponds to a grain-size class. The entries for each row are nonnegative and sum to one because each row represents proportions for each specimen across grain-size classes. The end-member unmixing problem can be formulated as
\begin{equation}\label{E: ema}
\begin{split}
     & \bm{P} \approx \bm{AS}\\
     \text{subject to} \quad & \bm{A} \in \mathbb{R}_+^{m \times K}, \bm{S} \in \mathbb{R}_+^{K \times n}, \bm{A1} = \bm{1}, \bm{S1} = \bm{1}. 
\end{split}
\end{equation}
Here, $K \le \min\{m,n\}$, each row of $\bm{S}$ represents an end-member, and each row of $\bm{A}$ contains the corresponding relative abundances of the end-members in a specimen. Both $\bm{A}$ and $\bm{S}$ satisfy nonnegative and row-sum-to-one constraints. The nonnegative and row-sum-to-one constraints on $\bm{S}$ follow from the fact that each end-member is itself a grain-size distribution, whereas the constraints on $\bm{A}$ follow from the interpretation of each specimen as a mixture of the end-members.

The end-member unmixing problem (\ref{E: ema}) is generally nonunique \citep{weltje1997end, paterson2015new, van2018genetically, qi2026nmf}. Traditional end-member unmixing methods, such as EMMA \citep{weltje1997end} and AnalySize \citep{seidel2015r}, seek end-members to be as close as possible to the observed specimens. However, highly mixed grain-size distribution data are common in nature, where true end-members do not lie near the observed specimens \citep{paterson2015new, van2018genetically, qi2026nmf, qi2026unmixing}. In such cases, traditional methods fail \citep{paterson2015new, van2018genetically}. Similar challenges also arise in other fields, including hyperspectral unmixing and magnetic mineral unmixing \citep{Bioucas62003622012, Heslop2012, Heslop2015}.  

To handle highly mixed data, some studies aim to identify end-members as distinct as possible \citep{van1999identifiability, ZHANG2020106656, qi2026unmixing, qi2026identification}. For example, \citet{van1999identifiability} applied a simulated annealing algorithm to maximize the chi-squared distances between end-members, while \citet{ZHANG2020106656} employed a genetic algorithm to maximize the sum of Manhattan distances between end-members and minimize the reconstruction error between $\bm{P}$ and $\bm{AS}$. However, simulated annealing and genetic algorithms are heuristic. More recently, \citet{qi2026unmixing, qi2026identification} investigated the use of volume-based measures to promote separation among end-members. Overall, methods for estimating end-members from highly mixed data remain limited in terms of model development and algorithmic advances.

Nonnegative matrix factorization (NMF) decomposes a nonnegative matrix into the product of two nonnegative matrices. This property makes NMF well suited for unmixing grain-size distribution data because its nonnegative constraint satisfies the end-member and abundance nonnegativity constraints in \eqref{E: ema} \citep{HESLOP200763, paterson2015new, qi2026identification}. Since the seminal work of \citet{lee1999learning} in \textit{Nature}, NMF has been widely applied in fields such as image processing, blind hyperspectral unmixing, and sedimentary geology \citep{van2018genetically, gillis2020nonnegative, GUO2024102379, SaberiMovahed2025Nonnegative, qi2026identification}.

This paper proposes a maximum-distance NMF model for highly mixed grain-size distribution data. The model maximizes the Euclidean distances between end-members and their centroid while imposing nonnegativity and row-sum-to-one constraints on both end-member matrix and abundance matrix. We adopt a hierarchical alternating least squares (HALS) to solve the resulting optimization problem and analyze its convergence properties. Experimental results demonstrate the potential of the proposed method for unmixing highly mixed grain-size distribution data.

The rest of the paper is organized as follows. Section~\ref{S: MAD-NMF} proposes the  maximum-distance NMF model. In
Section~\ref{S: Algorithm}, we provide the HALS algorithm and analyze its convergence.
Experimental settings and results are shown in Section~\ref{S: ES} and Section~\ref{S: ER}, respectively. Section~\ref{s: con} concludes this paper.

\section{Maximum-distance nonnegative matrix factorization}\label{S: MAD-NMF}

The objective function of the proposed maximum-distance nonnegative matrix factorization (MAD-NMF) is formulated as:
\begin{equation}\label{E: madnmfobjpos}
\begin{split}
     \text{min} \quad &||\bm{P} - \bm{A}\bm{S}||_{F}^2 - \lambda ||\bm{C}_{K\times K}\bm{S}||_F^2  \\
 \text{subject to} \quad   &  \bm{A}  \in \mathbb{R}_{+}^{m\times K}, \bm{S}  \in \mathbb{R}_{+}^{K\times n}, 
 \\& \bm{A}\bm{1}_{K\times 1} = \bm{1}_{m\times 1}, 
 \bm{S}\bm{1}_{n\times 1} = \bm{1}_{K\times 1},
\end{split}
\end{equation}
where $\lambda > 0$ and $\bm{C}_{K\times K} = \bm{I}_{K\times K} - \frac{1}{n}\bm{1}_{K\times 1}\bm{1}_{K\times 1}^T$. Here, $\bm{I}_{K\times K} \in \mathbb{R}^{K\times K}$ is an identity matrix and $\bm{1}_{K\times 1} \in \mathbb{R}^{K\times 1}$ is an all-one vector. For a given matrix $\bm{Y}$, $||\bm{Y}||_F^2 = \sum_i \sum_j \bm{Y}(i,j)^2$. 

The square $||\bm{P} - \bm{AS}||_F^2$ of the Frobenius norm is one of the most commonly used reconstruction errors between $\bm{P}$ and $\bm{AS}$, motivated by the assumption of Gaussian noise \citep{gillis2020nonnegative}. Alternative error measures, such as the Kullback-Leibler divergence, can also be adopted. For end-member distance regularization, we adopt $||\bm{C}_{K\times K}\bm{S}||_F^2$, which measures the difference between rows of $\bm{S}$ and their centroid: the $k$th row of matrix $\bm{C}_{K\times K}\bm{S}$ is  $\bm{S}(k, :) - \frac{1}{K}\sum_{k' = 1}^{K}\bm{S}(k', :)$ \citep{paterson2015new, Yu2012using, NUS2020104090}.

MAD-NMF uses $-\lambda$, which promotes larger separation among the end members and drives them away from the observed specimens. This property is essential for decomposing highly mixed grain-size distribution data, where no observed specimens are close to the true end members.

\section{Algorithm and convergence analysis}\label{S: Algorithm}

\subsection{Algorithm}
The problem in \eqref{E: madnmfobjpos} is non-convex. We employ the hierarchical alternating least squares (HALS) algorithm, and update one column of $\bm{A}$ or one row of $\bm{S}$ at a time \citep{paterson2015new, chen2012hals}.

Given $\bm{A}$ fixed, the subproblem for $\bm{S}$ is
\begin{equation}\label{E: madnmfobjposS}
\begin{split}
     \text{min} \quad &||\bm{P} - \bm{A}\bm{S}||_{F}^2  - \lambda||\bm{C}_{K\times K}\bm{S}||_F^2 \\
 \text{subject to} \quad   &  \bm{S}  \in \mathbb{R}_{+}^{K\times n}, \bm{S}\bm{1}_{n\times 1} = \bm{1}_{K\times 1}.
\end{split}
\end{equation}
Using the HALS strategy, we update one row $\bm{S}(k,:)$ of $\bm{S}$ at a time. Let $\bm{P}^{(k)} = \bm{P} - \sum_{l \neq k}\bm{A}(:, l)\bm{S}(l, :)$, where $\bm{A}(:, l)$ is the $l$-th column of $\bm{A}$ and $\bm{S}(l, :)$ is the $l$-th row of $\bm{S}$. Let $g(\bm{S}(k, :)) = ||\bm{P}^{(k)} - \bm{A}(:, k)\bm{S}(k, :)||_{F}^2 - \lambda \sum_{t = 1}^{K} ||\bm{S}(t, :)-\frac{1}{K}\left(\bm{S}(k, :) + \sum_{l \neq k}\bm{S}(l, :)\right)||_2^2$. Then the subproblem of \eqref{E: madnmfobjposS} becomes
\begin{equation*}\label{E: madnmfobjposShalsbe}
\begin{split}
     \text{min}\quad &g(\bm{S}(k, :))
     \\ 
      \text{subject to} \quad &  \bm{S}(k, :)  \in \mathbb{R}_{+}^{1\times n}, \bm{S}(k, :)\bm{1}_{n\times 1} = 1.
\end{split}
\end{equation*}
The Hessian matrix of $g(\bm{S}(k, :))$ is $(2\|\bm{A}(:,k)\|_2^2 - 2\lambda\left(1-\frac{1}{K}\right))\bm{I}_{n\times n}$, which is positive definite provided that $\lambda < \|\bm{A}(:,k)\|_2^2
/ \left(1-\frac{1}{K}\right)$.
 The gradient of $g(\bm{S}(k, :))$ with respect to $\bm{S}(k, :)$ is
\begin{equation*}
     \frac{\partial g(\bm{S}(k, :))}{\partial \bm{S}(k, :)}
     =2\bm{A}(:, k)^T\bm{A}(:, k)\bm{S}(k, :) - 2\bm{A}(:, k)^T\bm{P}^{(k)}
    -2\lambda \left(1-\frac{1}{K}\right)\bm{S}(k, :) + \frac{2\lambda}{K}\sum_{l \neq k}\bm{S}(l, :)
\end{equation*}
By setting $\frac{\partial g(\bm{S}(k, :))}{\partial \bm{S}(k, :)}$ to 0, the HALS update can be obtained as
\begin{equation}\label{E: madnmfobjposShals}
      \bm{S}(k, :) = 
     \text{Proj}_{\bm{S}(k, :) \geq 0, \bm{S}(k, :)\bm{1} = {1}}\left(\frac{\bm{A}(:, k)^T \bm{P}^{(k)} - \frac{\lambda}{K}\sum_{l \neq k}\bm{S}(l, :)}{||\bm{A}(:, k)||^2_2 - \lambda\left(1-\frac{1}{K}\right)}\right).
\end{equation}
\noindent Given a vector $\bm{y}^T\in \mathbb{R}^J$, $\text{Proj}_{\bm{y} \geq 0, \bm{y}\bm{1} = {1}}(\bm{y})$
represents Euclidean projection of $\bm{y}$ on probability simplex $\{\bm{y}^T \in \mathbb{R}^J | \bm{y} \geq 0, \bm{y}\bm{1} = {1}\}$, which has an unique solution. For detailed implementation, see \citet{Wang2013Projection}.

Given $\bm{S}$ fixed, the subproblem for $\bm{A}$ is
\begin{equation*}
\begin{split}
     \text{min} \quad &||\bm{P} - \bm{A}\bm{S}||_{F}^2\\
 \text{subject to} \quad   &  \bm{A}  \in \mathbb{R}_{+}^{m\times K}, \bm{A}\bm{1}_{K\times 1} = \bm{1}_{m\times 1}.
\end{split}
\end{equation*}
Following \citet{paterson2015new}, we enforce the row-sum-to-one constraint $\bm{A}\bm{1} = \bm{1}$  via a regularized term, yielding:
\begin{equation}\label{E: madnmfobjposAsumtoone}
\begin{split}
     \text{min} \quad &||\bm{P} - \bm{A}\bm{S}||_{F}^2 + \alpha ||\bm{A}\bm{1}_{K\times 1} - \bm{1}_{m\times 1}||_2^2\\
 \text{subject to} \quad   &  \bm{A}  \in \mathbb{R}_{+}^{m\times K}.
\end{split}
\end{equation}
This makes the Hessian matrix of the optimization problem positive definite as seen later. Using the HALS strategy, we update one column $\bm{A}(:,k)$ at a time. Let $f(\bm{A}(:, k)) = ||\bm{P}^{(k)} - \bm{A}(:, k)\bm{S}(k, :)||_{F}^2 +\alpha ||\bm{A}(:, k) + \sum_{l\neq k}\bm{A}(:, l) - \bm{1}_{m\times 1}||_2^2$. Then, the subproblem of \eqref{E: madnmfobjposAsumtoone} becomes
\begin{equation*}
\begin{split}
    \text{min} \quad &f(\bm{A}(:, k))\\
     \text{subject to} \quad   &  \bm{A}(:, k)  \in \mathbb{R}_{+}^{m\times 1}.
\end{split}
\end{equation*}
The Hessian matrix of $f(\bm{A}(:, k))$ is $(2||\bm{S}(k, :)||^2 + 2\alpha)\bm{I}_{m\times m}$, which is positive definite provided that $\alpha > 0$.
The gradient of $f(\bm{A}(:, k))$ with respect to $\bm{A}(:, k)$ is
\begin{equation*}
    \frac{\partial f(\bm{A}(:, k))}{\partial \bm{A}(:, k)}  
    = 2\left(\bm{A}(:, k)\bm{S}(k, :) - \bm{P}^{(k)}\right)\bm{S}(k, :)^T 
   + 2\alpha \left(\bm{A}(:, k) + \sum_{l\neq k}\bm{A}(:, l) - \bm{1}_{m\times 1}\right).
\end{equation*}
By setting $\frac{\partial f(\bm{A}(:, k))}{\partial \bm{A}(:, k)}$ to 0, the HALS update can be obtained as
\begin{equation}\label{E: madnmfobjposAhals}
       \bm{A}(:, k) = \text{max}\left\{0, \frac{\bm{P}^{(k)}\bm{S}(k, :)^T + \alpha (\bm{1}_{m\times 1} - \sum_{l \neq k}\bm{A}(:, l))}{||\bm{S}(k, :)||^2 + \alpha}\right\},
\end{equation}
where the max operation is performed componentwise.

The HALS procedure is summarized in Algorithm~\ref{alg: hals}. After HALS, a final fully constrained least squares refinement is applied to obtain abundance estimates $\bm{A}$ based on the estimated end-member matrix $\bm{S}$ and observed matrix $\bm{P}$ \citep{heinz2001fully}.

\begin{algorithm}
\caption{Hierarchical alternating least squares (HALS) for solving \eqref{E: madnmfobjpos} \citep{paterson2015new, chen2012hals}}
\label{alg: hals}
\KwIn{Input matrix $\bm{P}$, dimensionality $K$, number of iterations \textit{iter}, $\lambda$, and $\alpha$.}
\KwOut{$\bm{A}$ and $\bm{S}$.}

Generate initial matrix $\bm{S}$ using SISAL \citep{Bioucas5289072} and then initial $\bm{A}$ using fully constrained least squares based on the initial $\bm{S}$ and observed $\bm{P}$ \citep{heinz2001fully}; See \citet{paterson2015new} for details.

\For{$t = 1, 2, \ldots, \text{iter}$}{
\For{$k = 1, 2, \ldots, K$}{
Update $\bm{S}(k, :)$ using \eqref{E: madnmfobjposShals};

Update $\bm{A}(:, k)$ using \eqref{E: madnmfobjposAhals}
}}
\end{algorithm}

\subsection{Convergence analysis}

\textbf{Assumption 1}: There is a constant $c > 0$, such that in the iteration process, $||\bm{A}(:, k)||_2^2 \geq c$ for any $k = 1, \cdots, K$.

\noindent This assumption is reasonable because zero columns in $\bm{A}$ are not typically encountered in practice. In our experiments, the columns of $\bm{A}$ remain nonzero throughout the iterations.

HALS algorithm is a block coordinate descent method \citep{gillis2020nonnegative}. Following \citet{gillis2020nonnegative}, we adapt Proposition 2.7.1 from \citet{Bertsekas1999, Bertsekas1999correction} to obtain the following theoretical convergence analysis. 

  \begin{theorem}
   Under Assumption 1, assume $\alpha > 0$ and $0 < \lambda < c/(
1-1/K)$, it follows that the limit point of the iterates of Algorithm~\ref{alg: hals} is a stationary point of the optimization problem:
\begin{equation}\label{E: madnmfobjposasumto1}
\begin{split}
\text{min} \quad & \|\bm{P}-\bm{A}\bm{S}\|_F^2
-\lambda\|\bm{C}_{K\times K}\bm{S}\|_F^2
+\alpha\|\bm{A}\bm{1}_{K\times1}-\bm{1}_{m\times1}\|_2^2\\
 \text{subject to} \quad   &  \bm{A}  \in \mathbb{R}_{+}^{m\times K}, \bm{S}  \in \mathbb{R}_{+}^{K\times n},  \bm{S}\bm{1}_{n\times 1} = \bm{1}_{K\times 1}.
\end{split}
\end{equation} 
\end{theorem}

\begin{proof}
Denote $F(\bm{A}, \bm{S}) = \|\bm{P}-\bm{A}\bm{S}\|_F^2
-\lambda\|\bm{C}_{K\times K}\bm{S}\|_F^2
+\alpha\|\bm{A}\bm{1}_{K\times1}-\bm{1}_{m\times1}\|_2^2$.

1. The objective function $F(\bm{A},\bm{S})$ is continuously differentiable .

2. For updating $\bm{S}(k, :)$, $\{\bm{S}(k, :) | \bm{S}(k, :)  \in \mathbb{R}_{+}^{1\times n}, \bm{S}(k, :)\bm{1}_{n\times 1} = 1\}$ is a closed convex set; for updating $\bm{A}(:, k)$, $\{\bm{A}(:, k) | \bm{A}(:, k)  \in \mathbb{R}_{+}^{m\times 1}\}$ is a closed convex set.

3. As analyzed above, the Hessian matrix of the $k$-th $\bm{A}$-subproblem is $(2||\bm{S}(k, :)||^2 + 2\alpha)\bm{I}_{m\times m}$,
which is positive definite. Hence each $\bm{A}(:,k)$ subproblem is strictly convex and admits a unique minimizer. As analyzed above, the Hessian matrix of the $k$-th $\bm{S}$-subproblem is $(2\|\bm{A}(:,k)\|_2^2 - 2\lambda\left(1-\frac{1}{K}\right))\bm{I}_{n\times n}$,
which is positive definite. Hence each $\bm{S}(k, :)$ subproblem is strictly convex and admits a unique minimizer.

4. Since each HALS update exactly minimizes one block while keeping the remaining blocks fixed, the objective value is monotonically nonincreasing after each block update. 

Therefore, all conditions of the block coordinate descent convergence theorem \citep{gillis2020nonnegative, Bertsekas1999, Bertsekas1999correction} are satisfied.
\end{proof}

Note that the convergence analysis is established for the penalized
objective function in Problem~\eqref{E: madnmfobjposasumto1}, rather than
the original constrained formulation in Problem~\eqref{E: madnmfobjpos}.
This is because the equality constraint on $\bm A$ is incorporated into
the objective function through the quadratic penalty term during the
optimization process (see (\ref{E: madnmfobjposAsumtoone})).

\section{Experiments settings}\label{S: ES}

This section describes the generation of the artificial datasets and the evaluation metrics used in the experiments.

\subsection{Generation of artificial grain-size data}

\subsubsection{Highly mixed two-end-member dataset}

\citet{paterson2015new} used a synthetic highly mixed dataset to demonstrate that AnalySize fails to recover the underlying end members. We use the same data generation procedure as in \citet{qi2026unmixing} to generate a highly mixed dataset, which is based on \citet{paterson2015new}. This dataset is constructed from two end members, each defined by a lognormal distribution over 100 grain-size classes (Figure~\ref{F: twodatatrueems}). The abundances of the two end members for each of the 99 specimens are generated such that the abundance of one specimen is uniformly distributed between 0.13 and 0.87, with the other abundance equal to one minus the first. Consequently, no specimen is close to a pure end member. For convenience, we refer to this dataset as the highly mixed two-end-member dataset.

\begin{figure}[H]
\centering 
\begin{subfigure}[b]{0.35\linewidth} \includegraphics[width=1\textwidth]{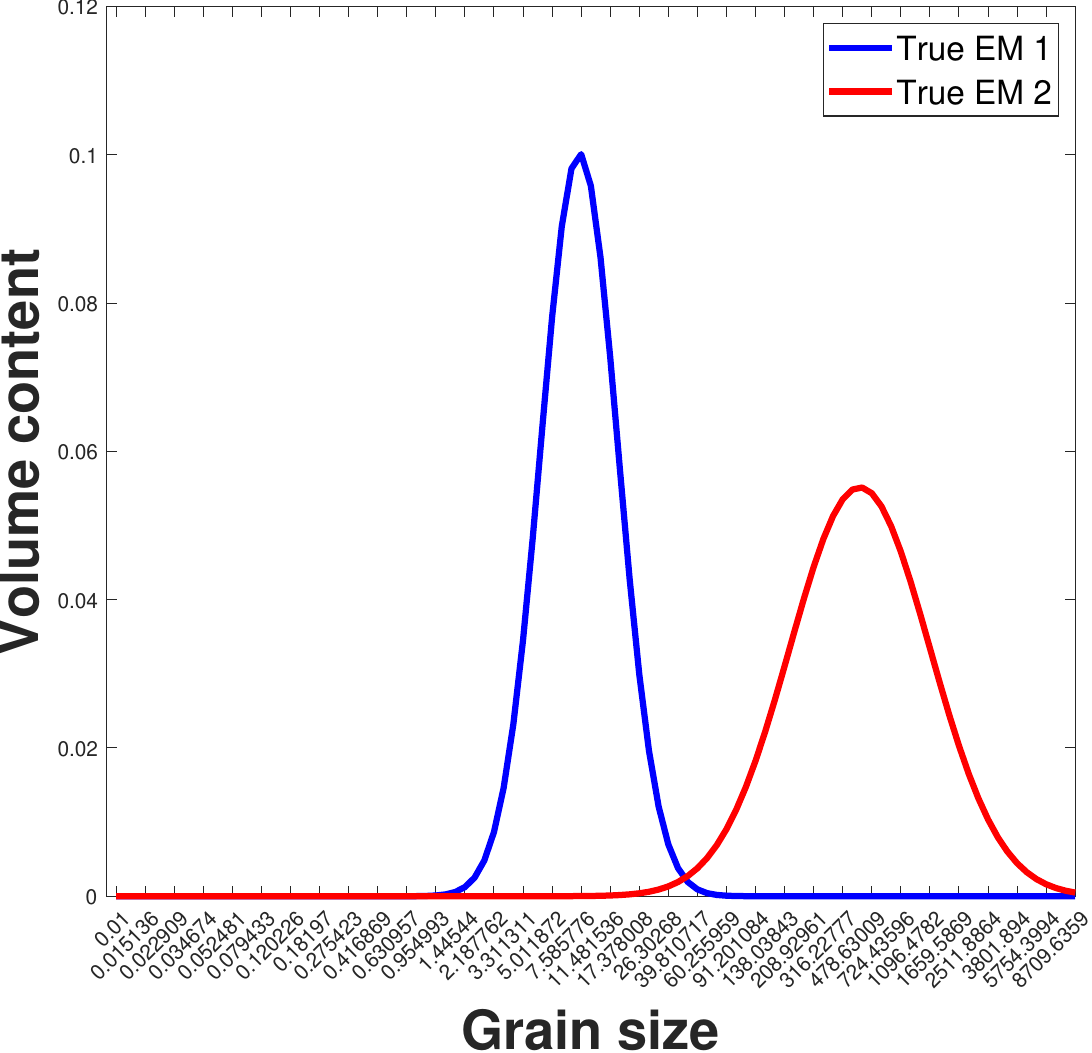}
\caption{Two-end-member dataset \citep{paterson2015new, qi2026unmixing}}\label{F: twodatatrueems} \end{subfigure} 
\begin{subfigure}[b]{0.35\linewidth} \includegraphics[width=1\textwidth]{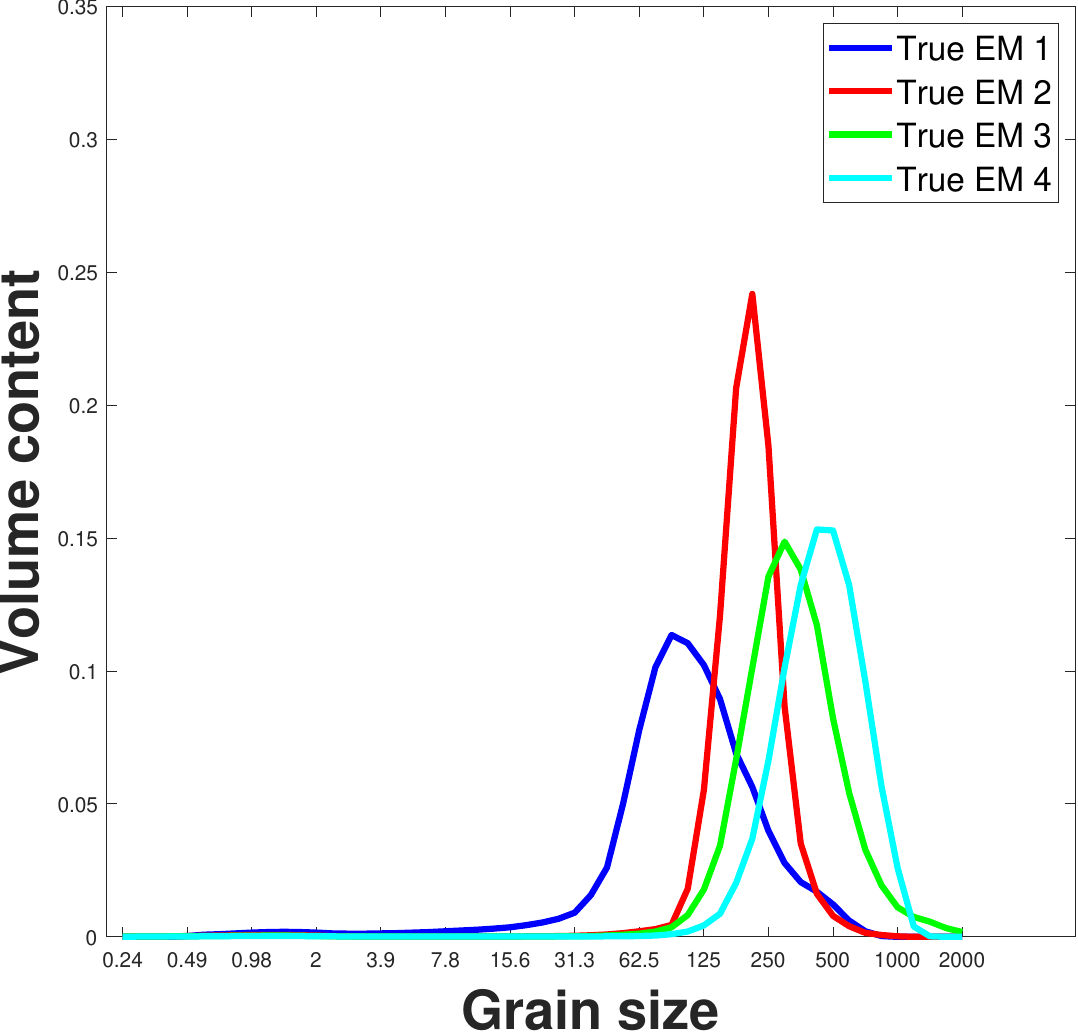} 
\caption{Coversand dataset \citep{van2018genetically, ZHANG2020106656}}\label{F: coversanddatatrueems}
\end{subfigure} 
\caption{True end members}\label{F: trueendmember} 
\end{figure}

In addition, a noisy version of the highly mixed two‑end‑member dataset is generated by multiplying the original data element‑wise by random numbers with mean 1 and standard deviation 0.01. Negative values are then set to zero, and each specimen is subsequently renormalized to sum to one.

\subsubsection{Highly mixed coversand dataset}

Following a similar data generation procedure to \citet{van2018genetically}, we generate a synthetic highly mixed coversand dataset consisting of 200 specimens by linearly mixing the four coversand end members reported by \citet{ZHANG2020106656}.  For each specimen, the abundances are generated uniformly within the interval [0.08,0.49] and accepted only if they sum to one. Consequently, all abundances satisfy the constraints of being nonnegative, summing to one, and lying between 0.08 and 0.49.

Again, to generate a noisy version of the highly mixed coversand dataset, we multiply the original data element-wise by random numbers with mean 1 and standard deviation 0.01. Negative values are then set to zero, and each specimen is subsequently renormalized to sum to one.
    
\subsection{Evaluation}

Performance is evaluated using the mean
angle between the estimated and true end members (MAEM), mean angle between the estimated and true abundances (MAAB) together with visual inspection \citep{paterson2015new, qi2026unmixing}. The MAEM and MAAB are defined as \citep{qi2026unmixing}
\begin{equation*}
    \text{MAEM} = \frac{1}{K}\sum_k\left(\frac{180}{\pi}\text{arccos}\left(\frac{\bm{S}(k, :)\hat{\bm{S}}(k, :)^T}{||\bm{S}(k, :)||_2||\hat{\bm{S}}(k, :)||_2}\right)\right)
\end{equation*}
and
\begin{equation*}
    \text{MAAB} = \frac{1}{m}\sum_i\left(\frac{180}{\pi}\text{arccos}\left(\frac{\bm{A}(i, :)\hat{\bm{A}}(i, :)^T}{||\bm{A}(i, :)||_2||\hat{\bm{A}}(i, :)||_2}\right)\right)
\end{equation*}
where $\bm{S}(k, :)$ and $\hat{\bm{S}}(k, :)$ are the $k$th true and estimated end members, respectively, and $\bm{A}(i, :)$ and $\hat{\bm{A}}(i, :)$ are the true and estimated abundances of the $i$th specimen, respectively. Smaller values of MAEM and MAAB indicate better performance.

Although the coefficient of determination ($R^2$) is commonly used to evaluate grain-size distribution unmixing models, it measures the reconstruction accuracy of $\bm{P}$ rather than the accuracy of the estimated end members and abundances \citep{van2018genetically, paterson2015new, ZHANG2020106656}. Therefore, we do not use $R^2$ as an evaluation metric in this study.

\section{Experimental results}\label{S: ER}

In this section, we compare AnalySize with the proposed maximum-distance NMF (MAD-NMF) on highly mixed grain-size distribution datasets and demonstrate that MAD-NMF effectively recovers the underlying end members and abundances. We choose AnalySize as the baseline because \citet{van2018genetically} reported that it outperformed several existing methods. The code for this paper is implemented in MATLAB R2024b and is available on \citet{qi2026madnmfcode}.

\subsection{Highly mixed two-end-member dataset}

We apply AnalySize and MAD-NMF to the highly mixed two-end-member dataset $\bm{P}$ to estimate the abundance matrix $\bm{A}$ and end-member matrix $\bm{S}$. 

Figures~\ref{F: twodataminems} and~\ref{F: twodataminabundances} present the estimated end members and abundances obtained by AnalySize, respectively. The results show that AnalySize fails to recover the true end members and abundances, which is consistent with the findings in \citet{paterson2015new, van2018genetically}. This failure occurs because AnalySize encourages the estimated end members to lie closer to one another and, consequently, closer to the observed specimens, whereas the true end members are located farther from the observed specimens.

Figures~\ref{F: twodatamaxems} and~\ref{F: twodatamaxabundances} present the estimated end members and abundances obtained by MAD-NMF, respectively. In contrast to AnalySize, MAD-NMF effectively recovers the end members and abundances. This is because MAD-NMF encourages greater separation among the end members, allowing them to be located farther from the observed specimens. Consequently, MAD-NMF is well-suited for highly mixed data. 

\begin{figure}[H]
\centering
\begin{subfigure}[b]{0.25\linewidth}
\includegraphics[width=1\textwidth]{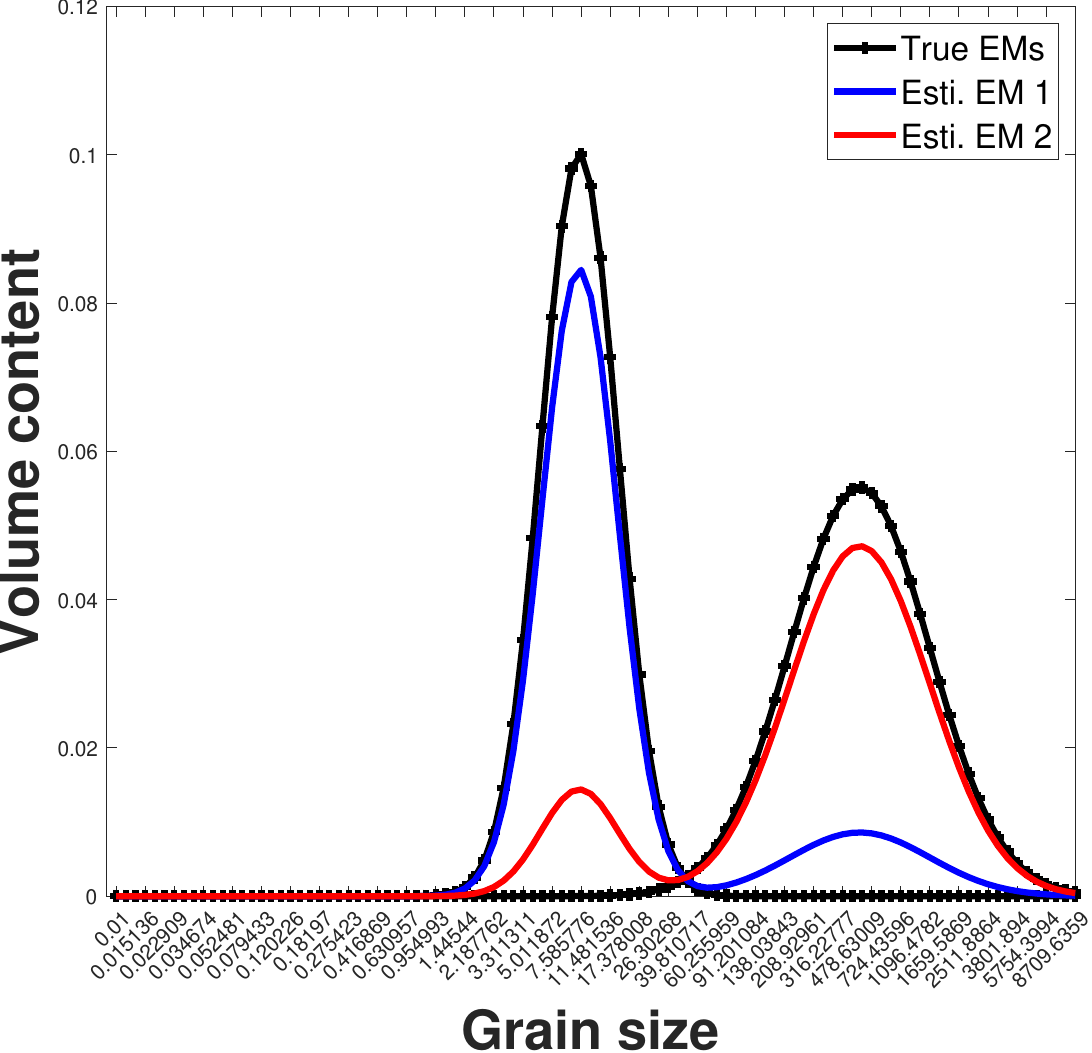}
    \caption{AnalySize: End Members}
    \label{F: twodataminems}
 \end{subfigure}
    \begin{subfigure}[b]{0.24\linewidth}
\includegraphics[width=1\textwidth]{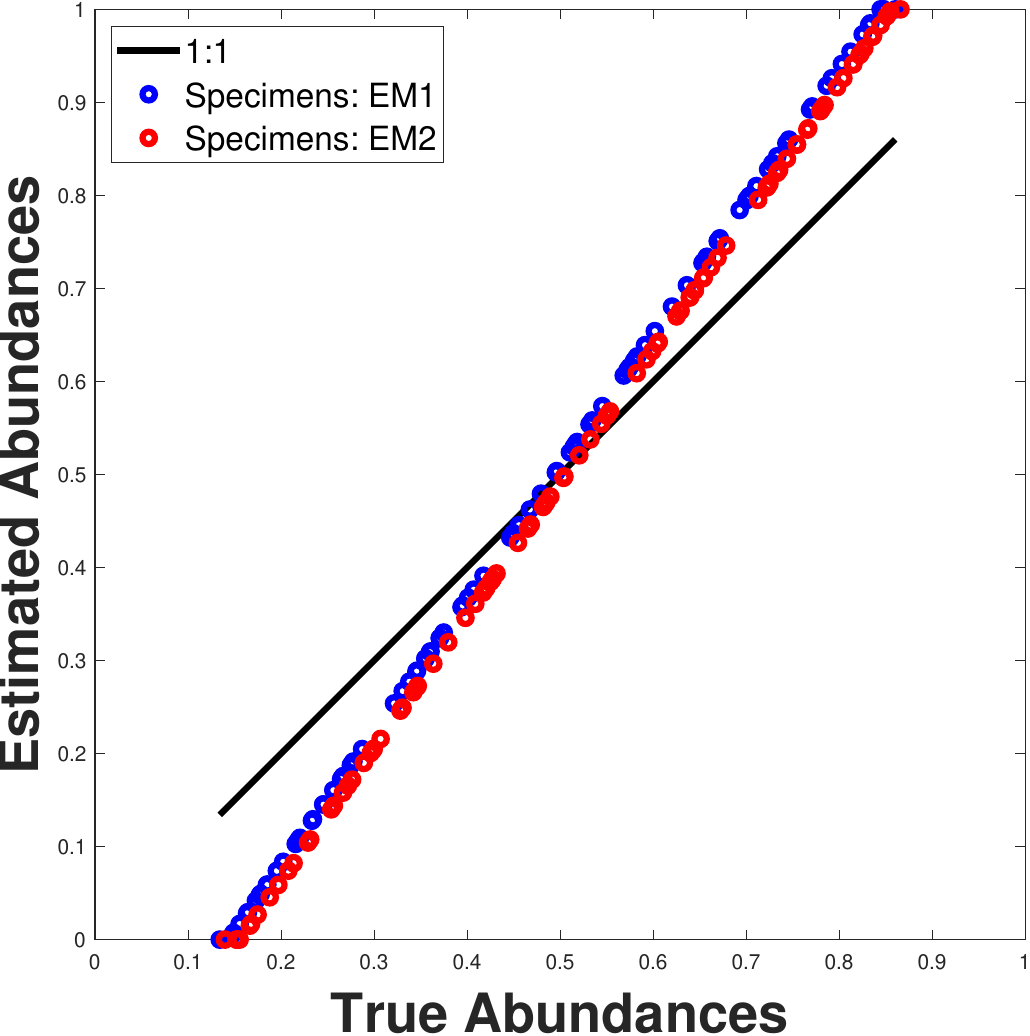}
    \caption{AnalySize: Abundances}
    \label{F: twodataminabundances}
   \end{subfigure} 
 \begin{subfigure}[b]{0.25\linewidth}
\includegraphics[width=1\textwidth]{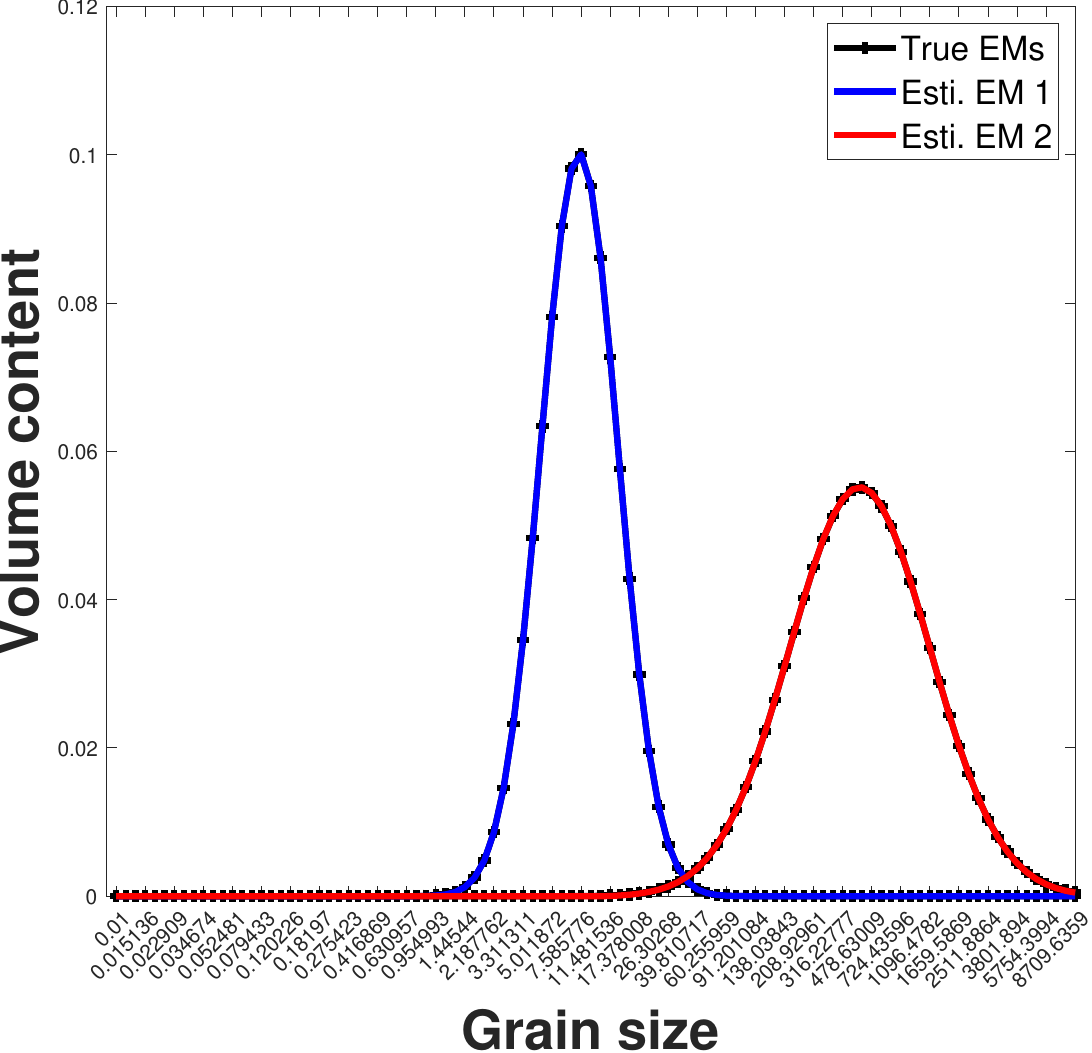}
    \caption{MAD-NMF: End Members}
    \label{F: twodatamaxems}
 \end{subfigure}
    \begin{subfigure}[b]{0.24\linewidth}
\includegraphics[width=1\textwidth]{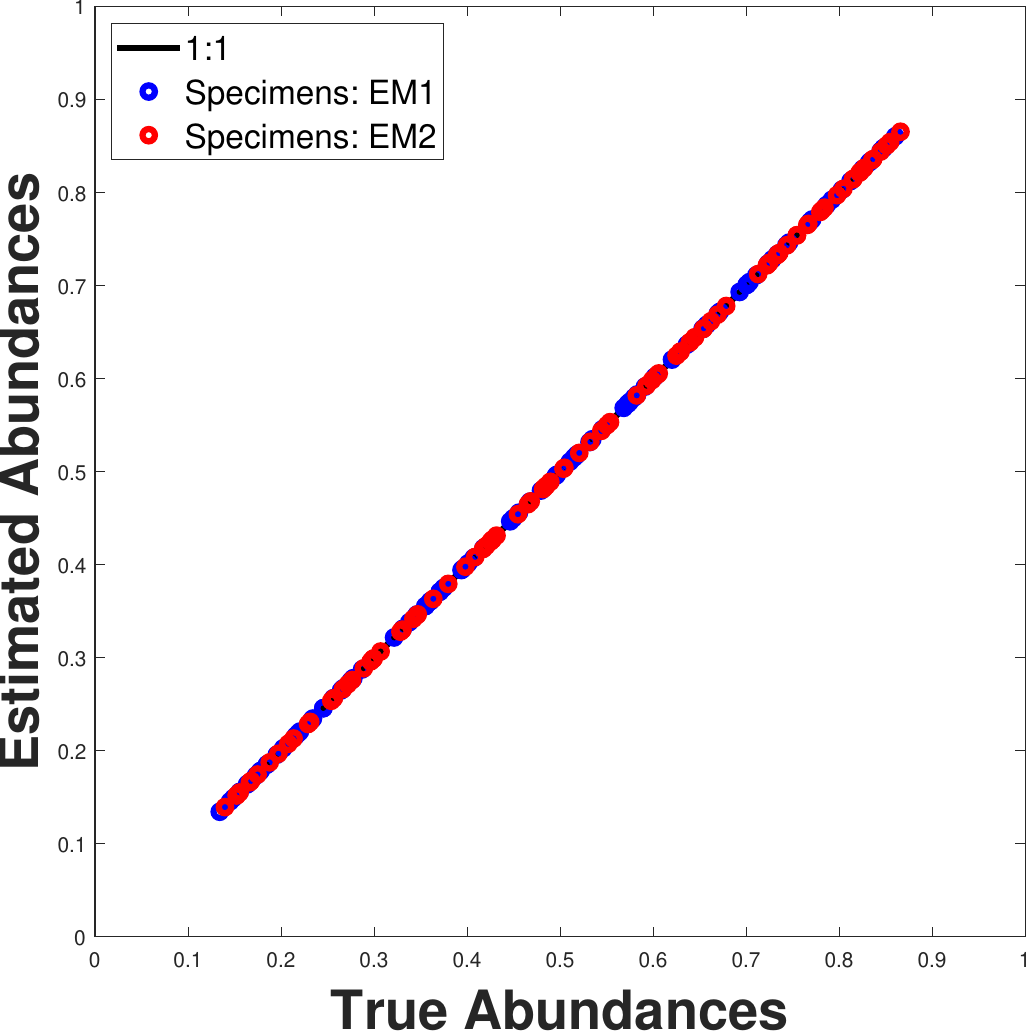}
    \caption{MAD-NMF: Abundances}
    \label{F: twodatamaxabundances}
 \end{subfigure}
 \caption{Highly mixed two-end-member dataset: $\alpha = 5$, $\lambda = 1$ (a and b) AnalySize; (c and d) MAD-NMF.}\label{F: twodata}
    \end{figure}

\begin{figure}[H]
\centering
\begin{subfigure}[b]{0.25\linewidth}
\includegraphics[width=1\textwidth]{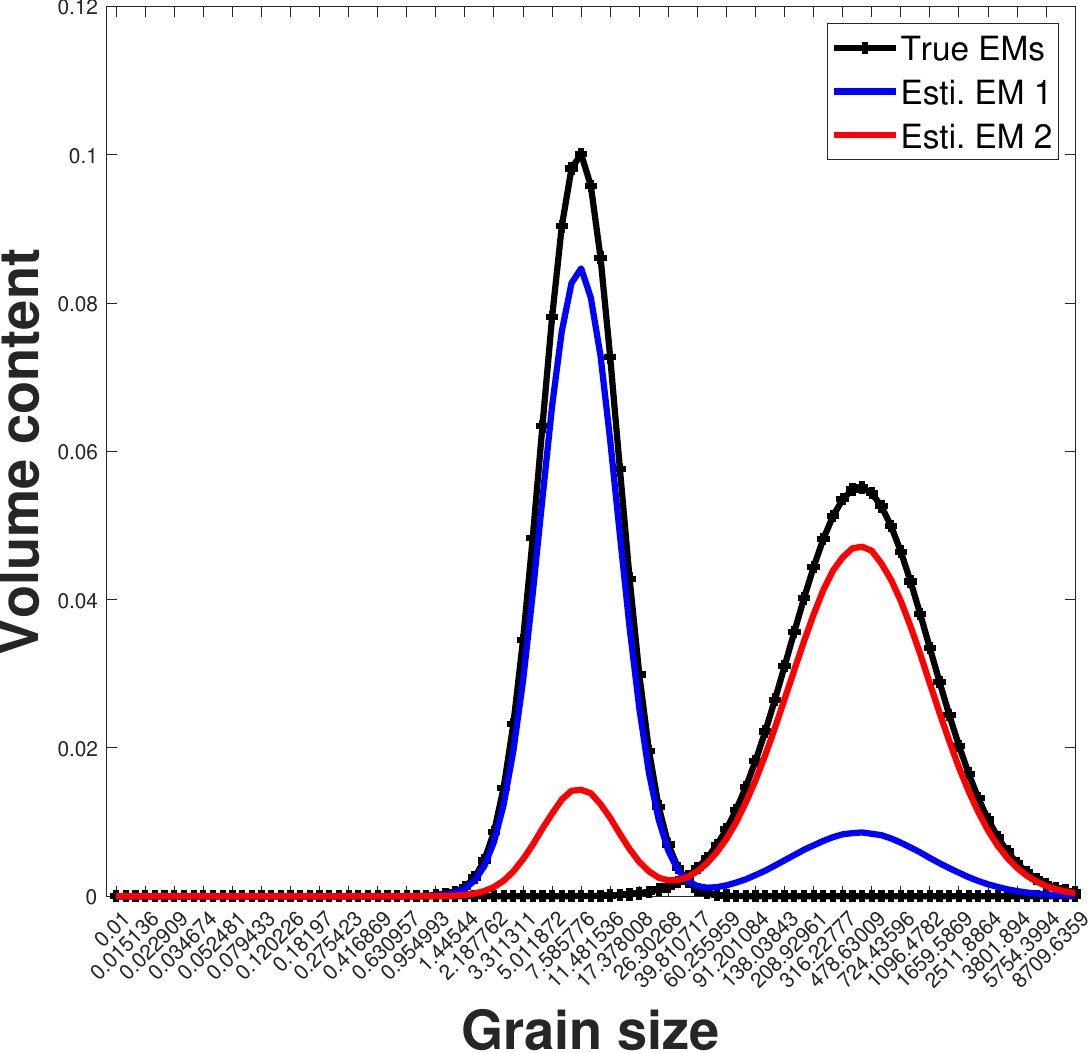}
    \caption{AnalySize: End Members}
    \label{F: noisytwodataminems}
 \end{subfigure}
    \begin{subfigure}[b]{0.24\linewidth}
\includegraphics[width=1\textwidth]{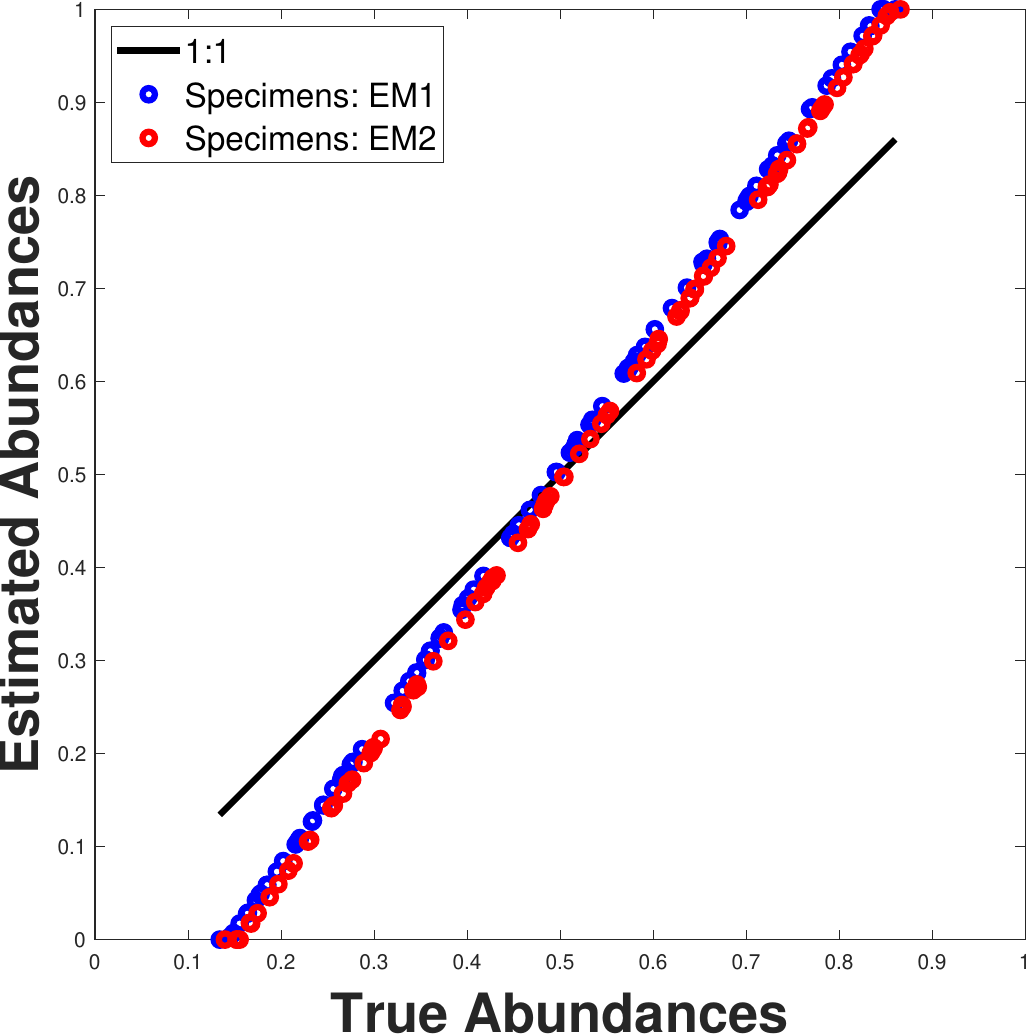}
    \caption{AnalySize: Abundances}
    \label{F: noisytwodataminabundances}
   \end{subfigure} 
 \begin{subfigure}[b]{0.25\linewidth}
\includegraphics[width=1\textwidth]{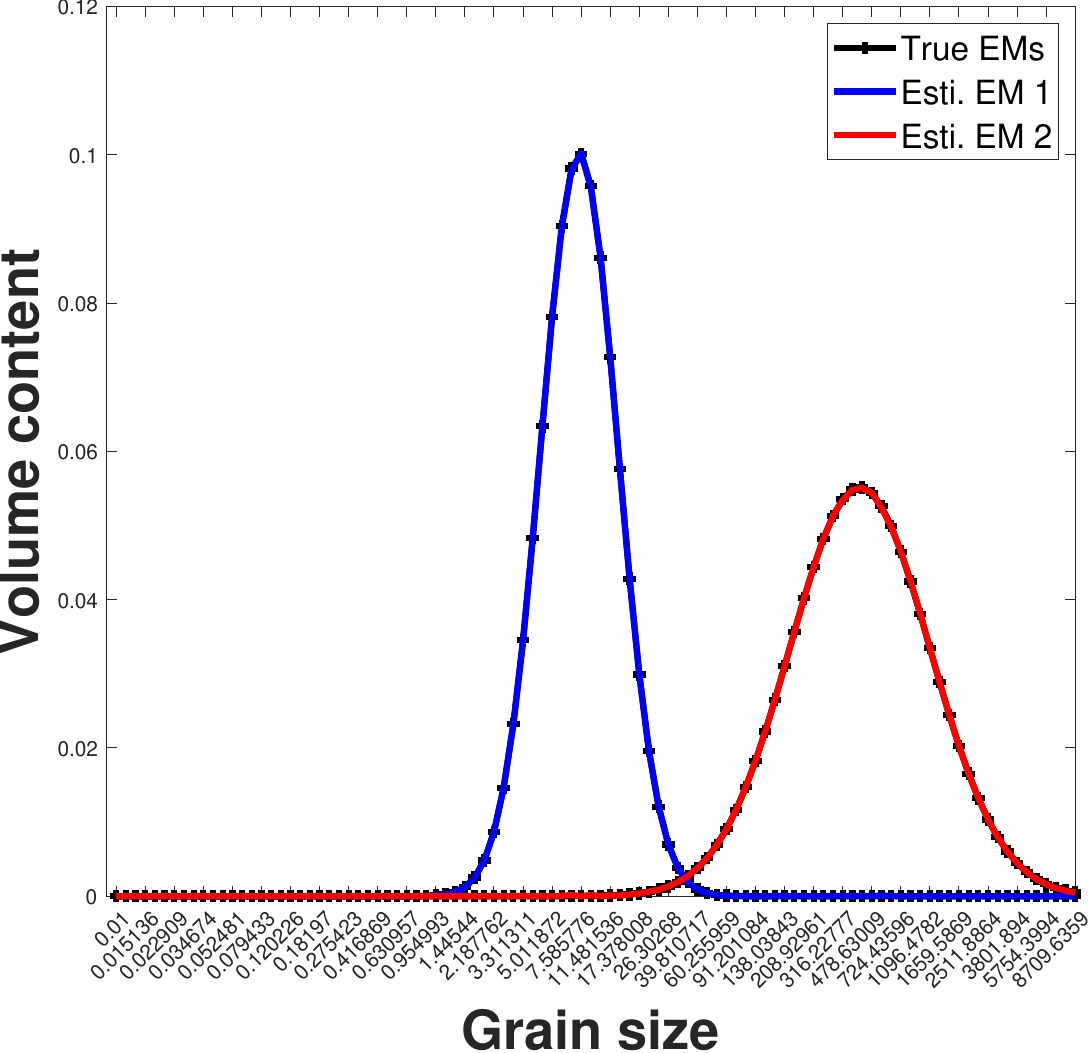}
    \caption{MAD-NMF: End Members}
    \label{F: noisytwodatamaxems}
 \end{subfigure}
    \begin{subfigure}[b]{0.24\linewidth}
\includegraphics[width=1\textwidth]{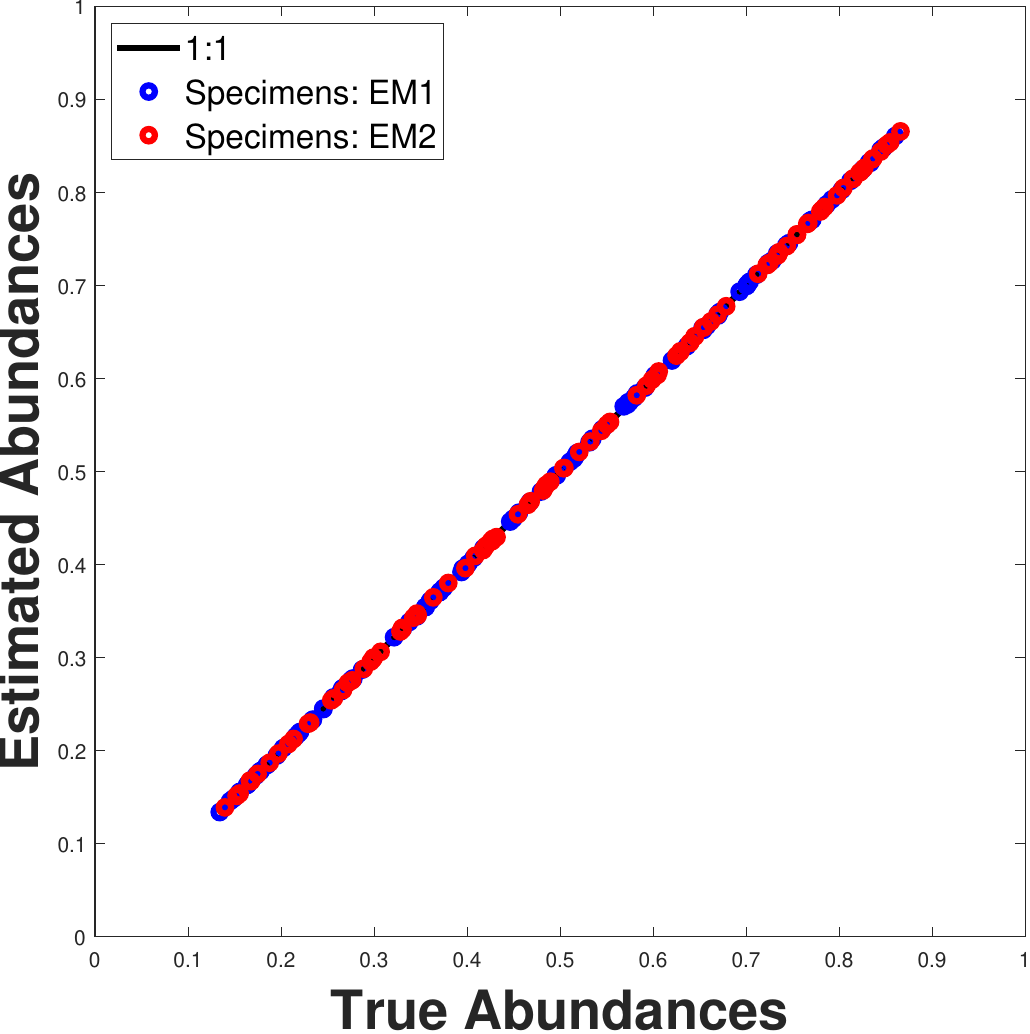}
    \caption{MAD-NMF: Abundances}
    \label{F: noisytwodatamaxabundances}
 \end{subfigure}
 \caption{Noisy highly mixed two-end-member dataset: $\alpha = 5$, $\lambda = 1$ (a and b) AnalySize; (c and d) MAD-NMF.}\label{F: noisytwodata}
    \end{figure}
    
Figure~\ref{F: noisytwodata} shows the estimated end members and abundances obtained by AnalySize and MAD-NMF for the noisy highly mixed two-end-member dataset. The results are similar to those obtained for the noise-free dataset.

Table~\ref{T: twosourcesangle} reports the MAEM and MAAB. MAD-NMF achieves lower values than AnalySize, demonstrating its superior ability to recover the end members and abundances, which is consistent with the visual results shown above.

\begin{table}[H]
\caption{MAEM and MAAB for (noisy) highly mixed two-end-member dataset}
    \label{T: twosourcesangle}
    \centering
    \begin{tabular}{rrr}
    \hline 
    Methods & AnalySize & MAD-NMF\\
    \hline
Highly mixed two-end-member dataset &&\\
 MAEM & 10.2711  &  0.0245  \\
 MAAB &   6.8464    & 0.0303\\
 \hline
Noisy highly mixed two-end-member dataset &&\\
 MAEM & 10.2701 &   0.1186
 \\
 MAAB &  6.8387 &   0.0609
\\
    \hline
    \end{tabular}
\end{table}

Table~\ref{T: twosourcesdistanceterm} reports the distances between the estimated end-members and their centroid.
The results show that MAD-NMF yields a larger distance than AnalySize. This is consistent with their respective objectives, namely that MAD-NMF maximizes the distances between end-members and their centroid while AnalySize minimizes them.

\begin{table}[H]
\caption{The $||\bm{C}_{K\times K}\bm{S}||_F^2$ for (noisy) highly mixed two-end-member dataset}
    \label{T: twosourcesdistanceterm}
    \centering
    \begin{tabular}{rrr}
    \hline 
    Methods & AnalySize & MAD-NMF\\
    \hline
   Highly mixed two-end-member dataset &     0.0269  &   0.0549
 \\
  Noisy highly mixed two-end-member dataset & 0.0269   & 0.0548
 \\
    \hline
    \end{tabular}
\end{table}

\subsection{Highly mixed coversand dataset}

As in the highly mixed two-end-member dataset, we apply AnalySize and MAD-NMF to the highly mixed coversand dataset $\bm{P}$ to estimate the abundance matrix $\bm{A}$ and end-member matrix $\bm{S}$. Figure~\ref{F: coversanddata} presents the estimated end members and abundances obtained by AnalySize and MAD-NMF. Once again, AnalySize fails to recover the true end members and abundances, which is consistent with the findings of \citet{paterson2015new, van2018genetically}. In contrast, MAD-NMF effectively recovers the true end members and abundances. Similar results are obtained for the noisy higly mixed coversand dataset, as shown in Figure~\ref{F: noisycoversanddata}. However, the abundance results from MAD-NMF in the noisy case (see Figure~\ref{F: noisycoversanddatamaxabundances}) are more
scattered than in the noise-free case (See Figure~\ref{F: coversanddatamaxabundances}).

\begin{figure}[H]
\centering
\begin{subfigure}[b]{0.25\linewidth}
\includegraphics[width=1\textwidth]{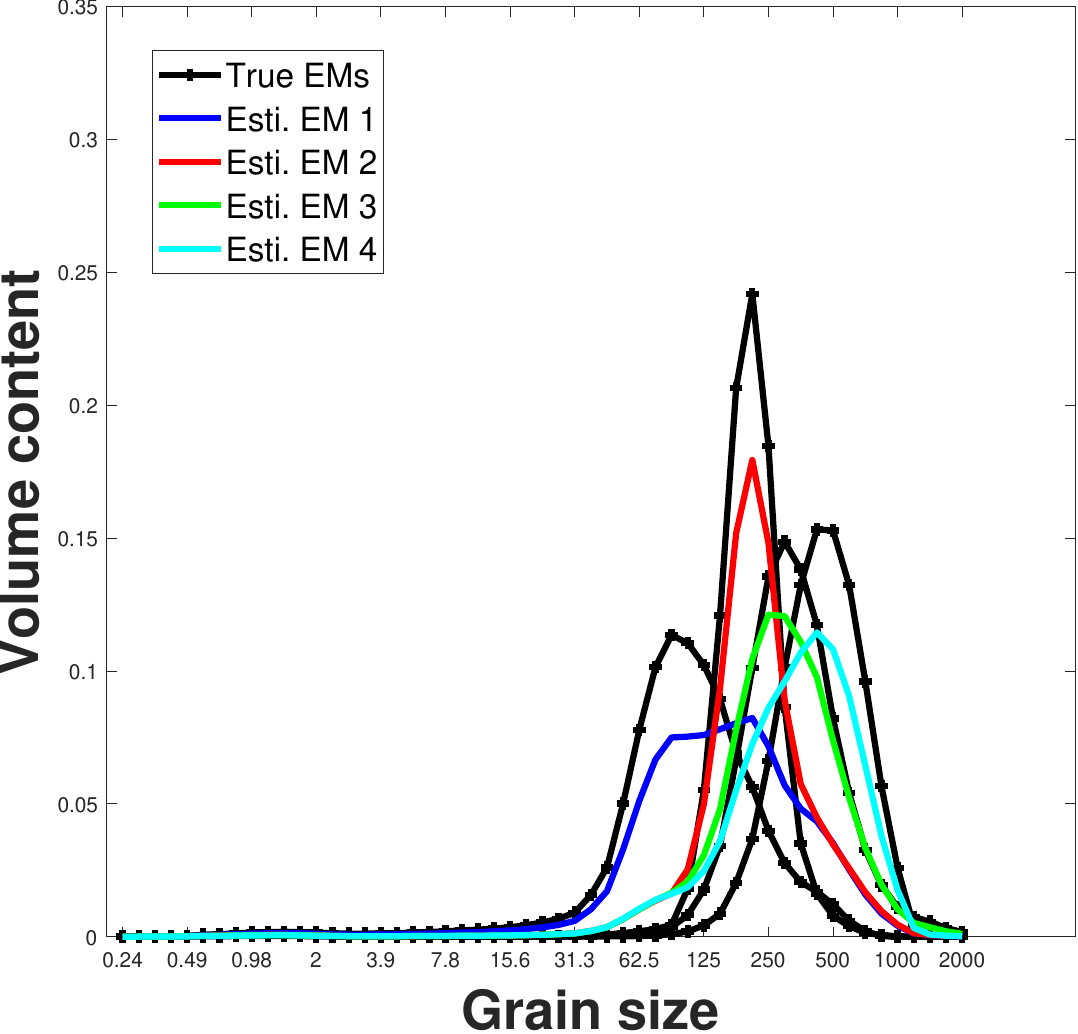}
    \caption{AnalySize: End Members}
    \label{F: coversanddataminems}
 \end{subfigure}
    \begin{subfigure}[b]{0.24\linewidth}
\includegraphics[width=1\textwidth]{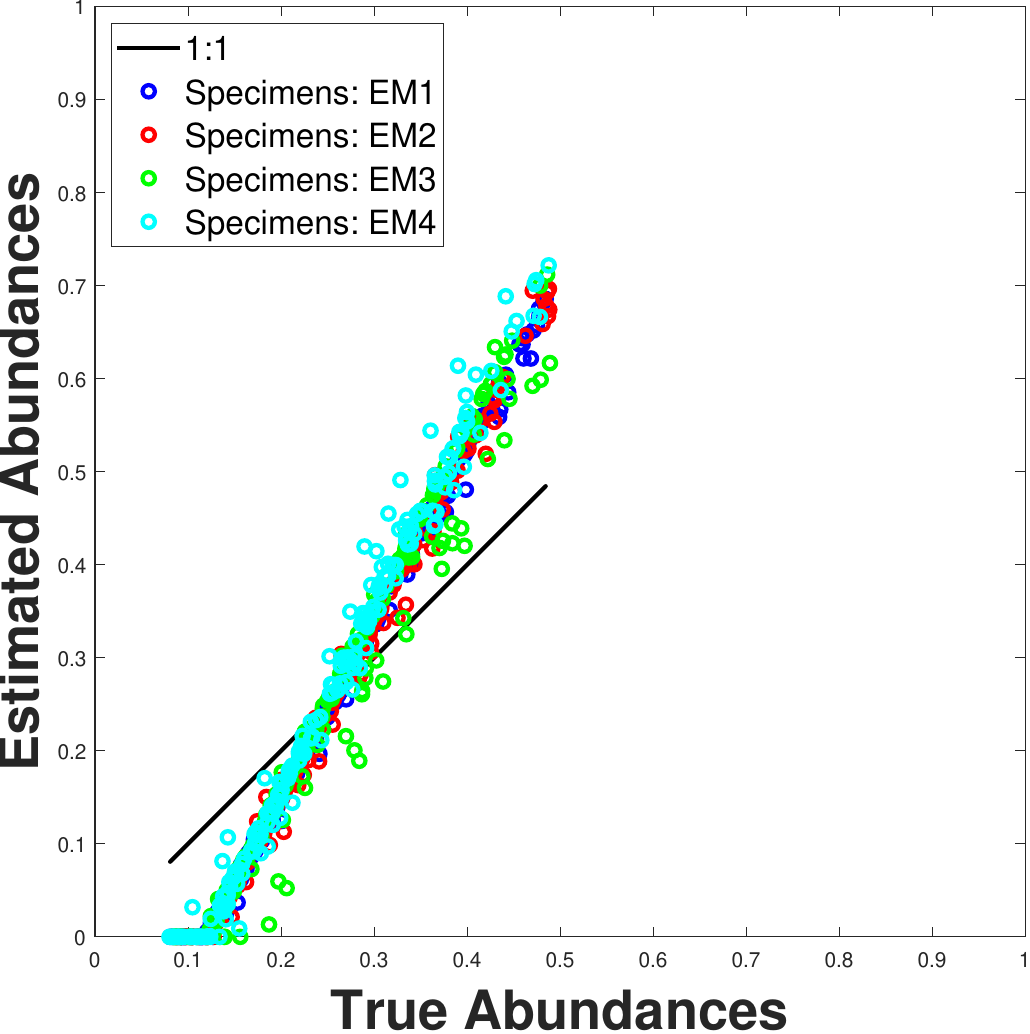}
    \caption{AnalySize: Abundances}
    \label{F: coversanddataminabundances}
   \end{subfigure} 
 \begin{subfigure}[b]{0.25\linewidth}
\includegraphics[width=1\textwidth]{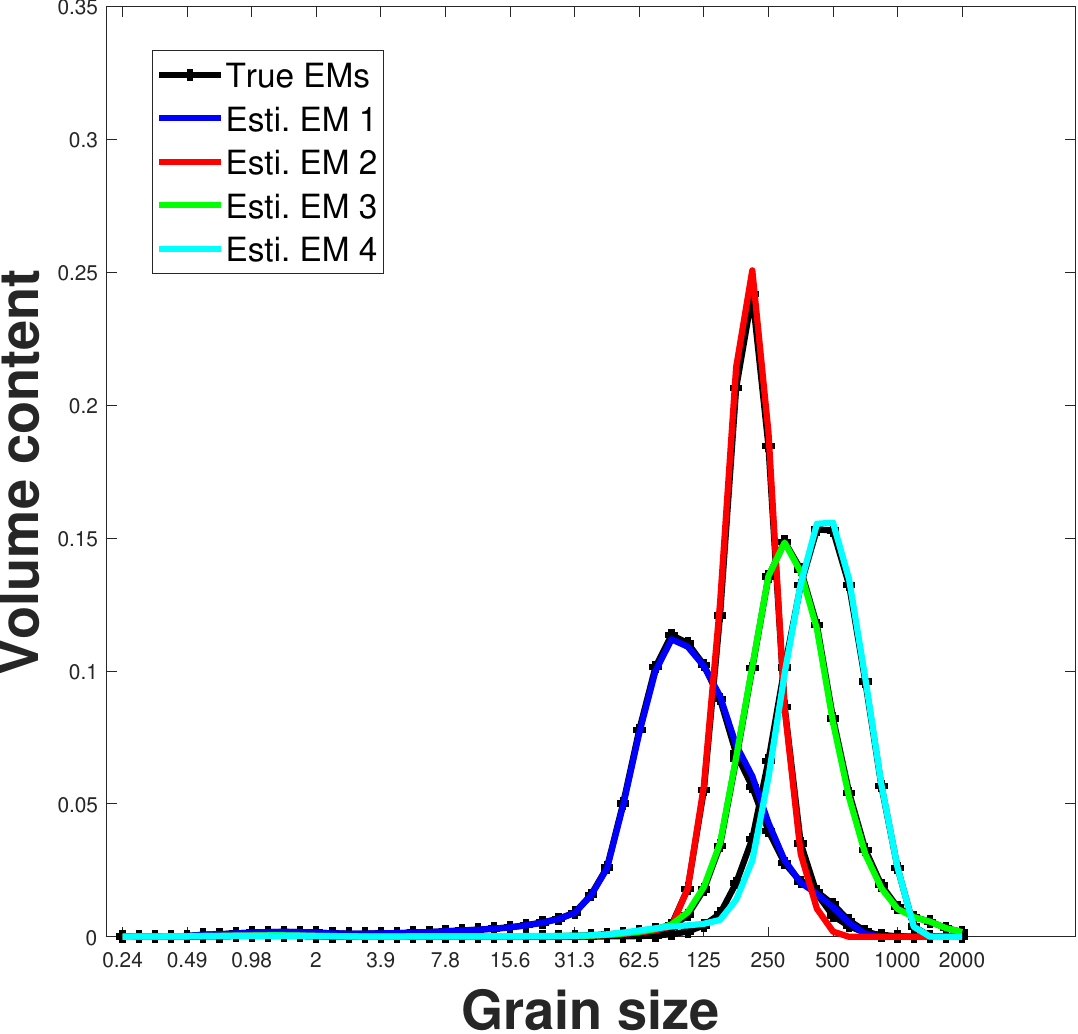}
    \caption{MAD-NMF: End Members}
    \label{F: coversanddatamaxbasis}
 \end{subfigure}
    \begin{subfigure}[b]{0.24\linewidth}
\includegraphics[width=1\textwidth]{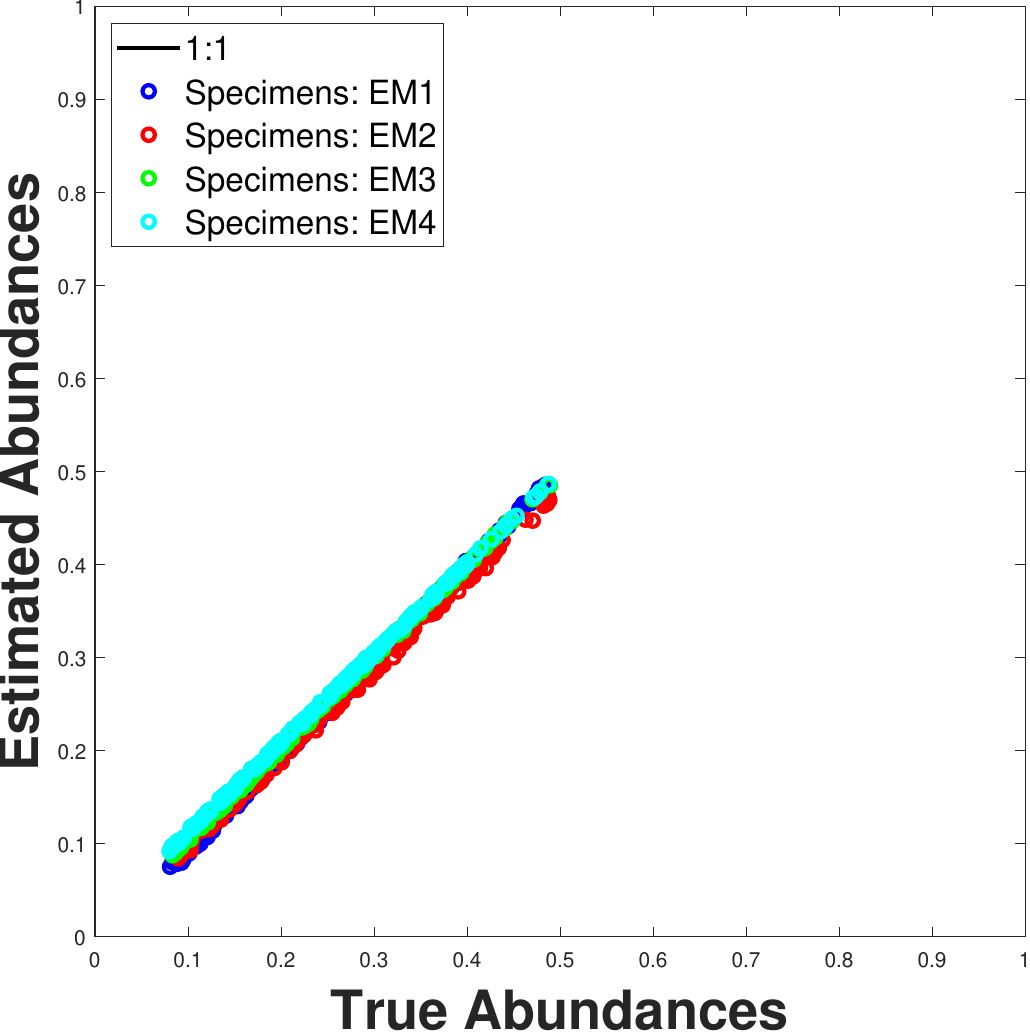}
    \caption{MAD-NMF: Abundances}
    \label{F: coversanddatamaxabundances}
 \end{subfigure}
 \caption{Highly mixed coversand dataset: $\alpha = 100$, $\lambda = 2.6$ (a and b) AnalySize; (c and d) MAD-NMF.}\label{F: coversanddata}
    \end{figure}

\begin{figure}[H]
\centering
\begin{subfigure}[b]{0.25\linewidth}
\includegraphics[width=1\textwidth]{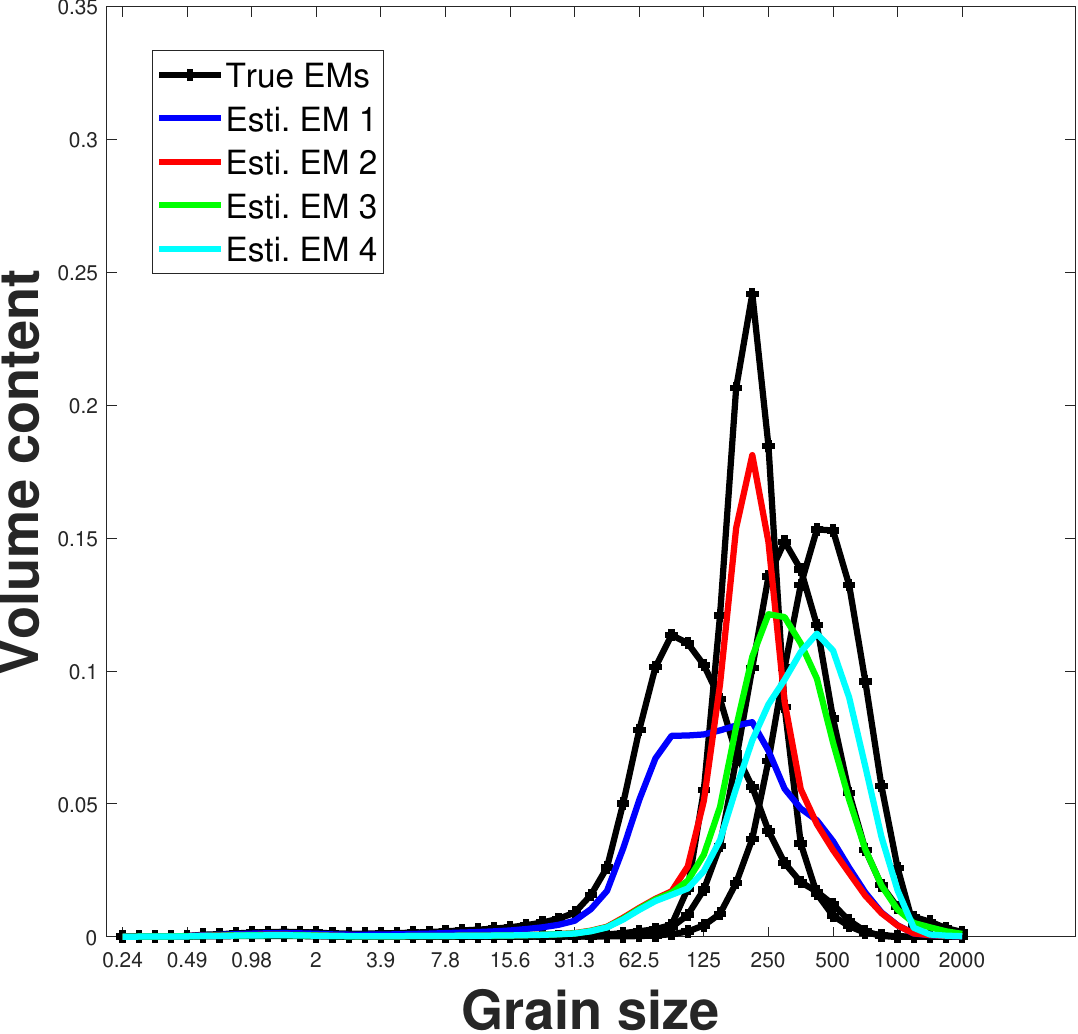}
    \caption{AnalySize: End Members}
    \label{F: noisycoversanddataminems}
 \end{subfigure}
    \begin{subfigure}[b]{0.24\linewidth}
\includegraphics[width=1\textwidth]{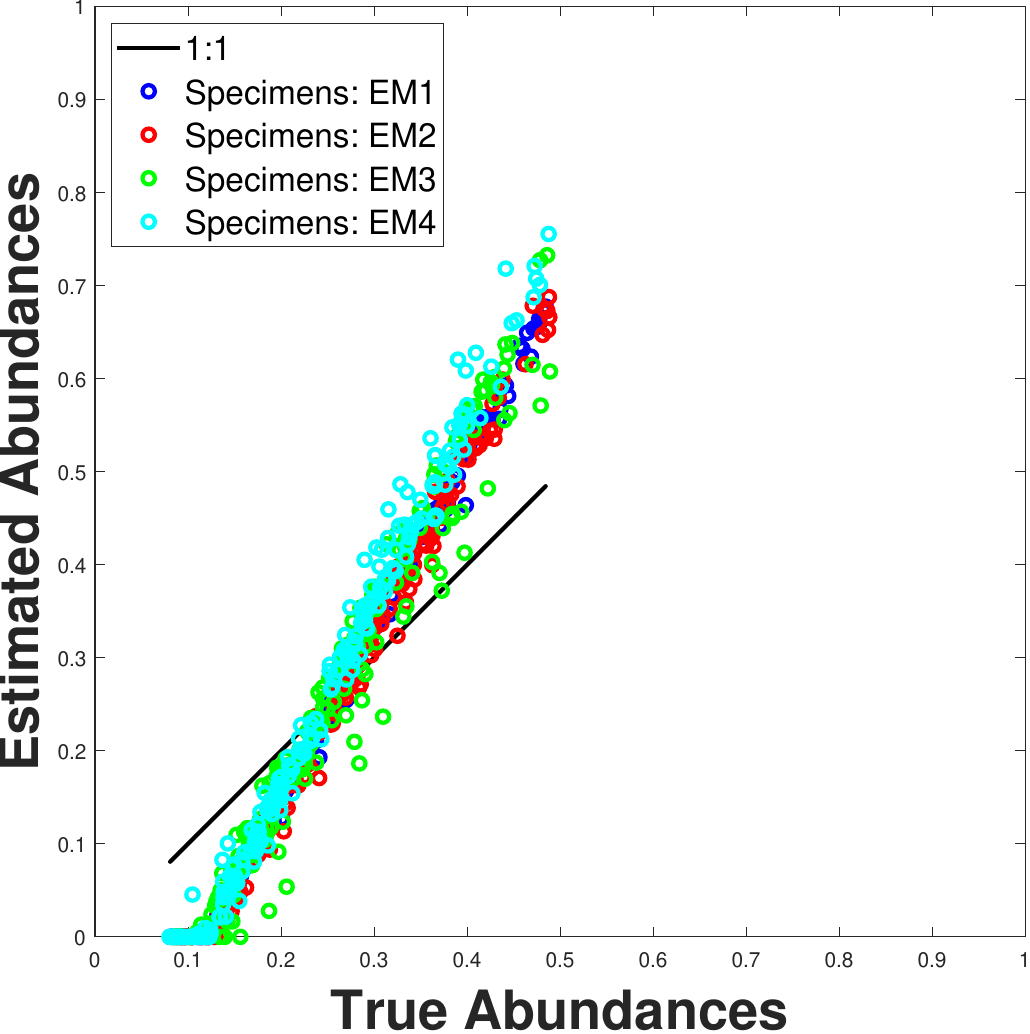}
    \caption{AnalySize: Abundances}
    \label{F: noisycoversanddataminabundances}
   \end{subfigure} 
 \begin{subfigure}[b]{0.25\linewidth}
\includegraphics[width=1\textwidth]{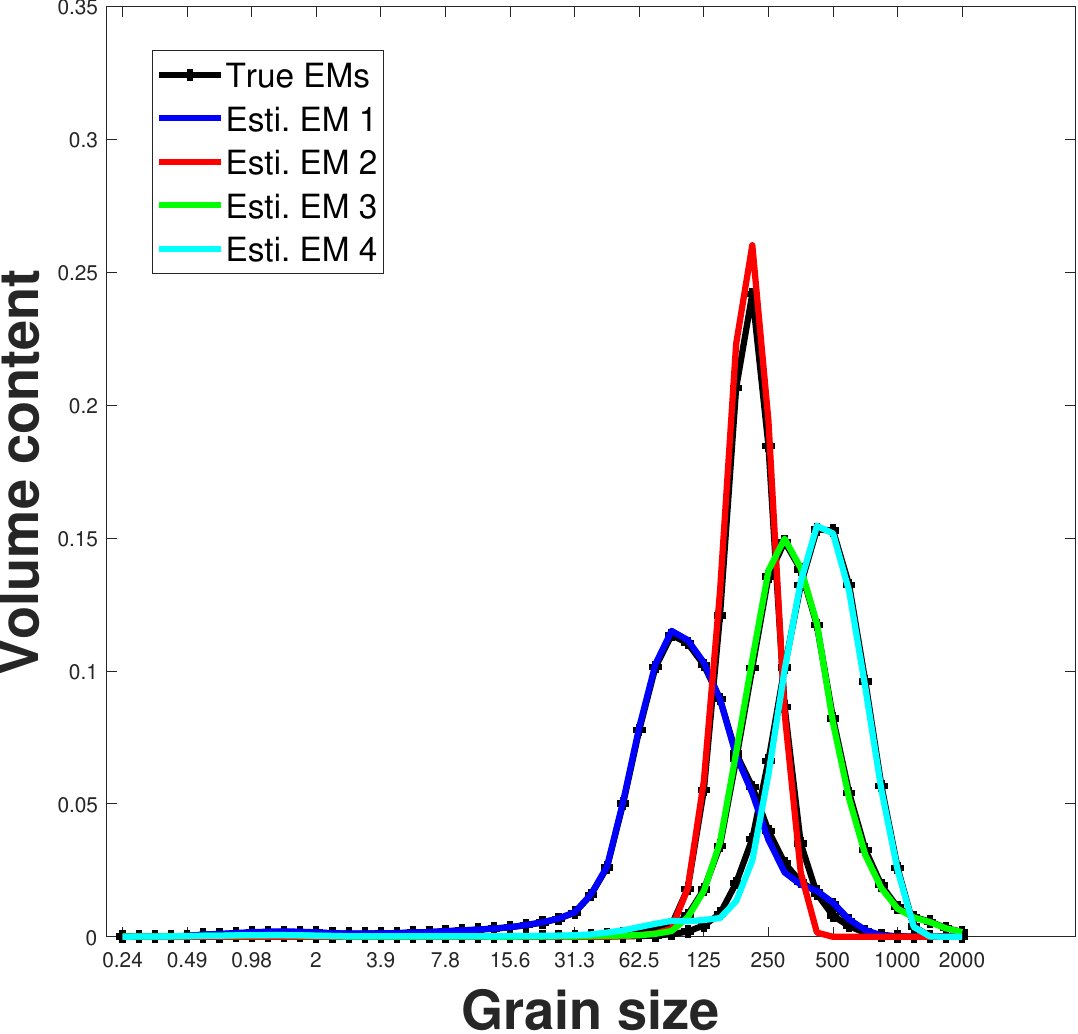}
    \caption{MAD-NMF: End Members}
    \label{F: noisycoversanddatamaxbasis}
 \end{subfigure}
    \begin{subfigure}[b]{0.24\linewidth}
\includegraphics[width=1\textwidth]{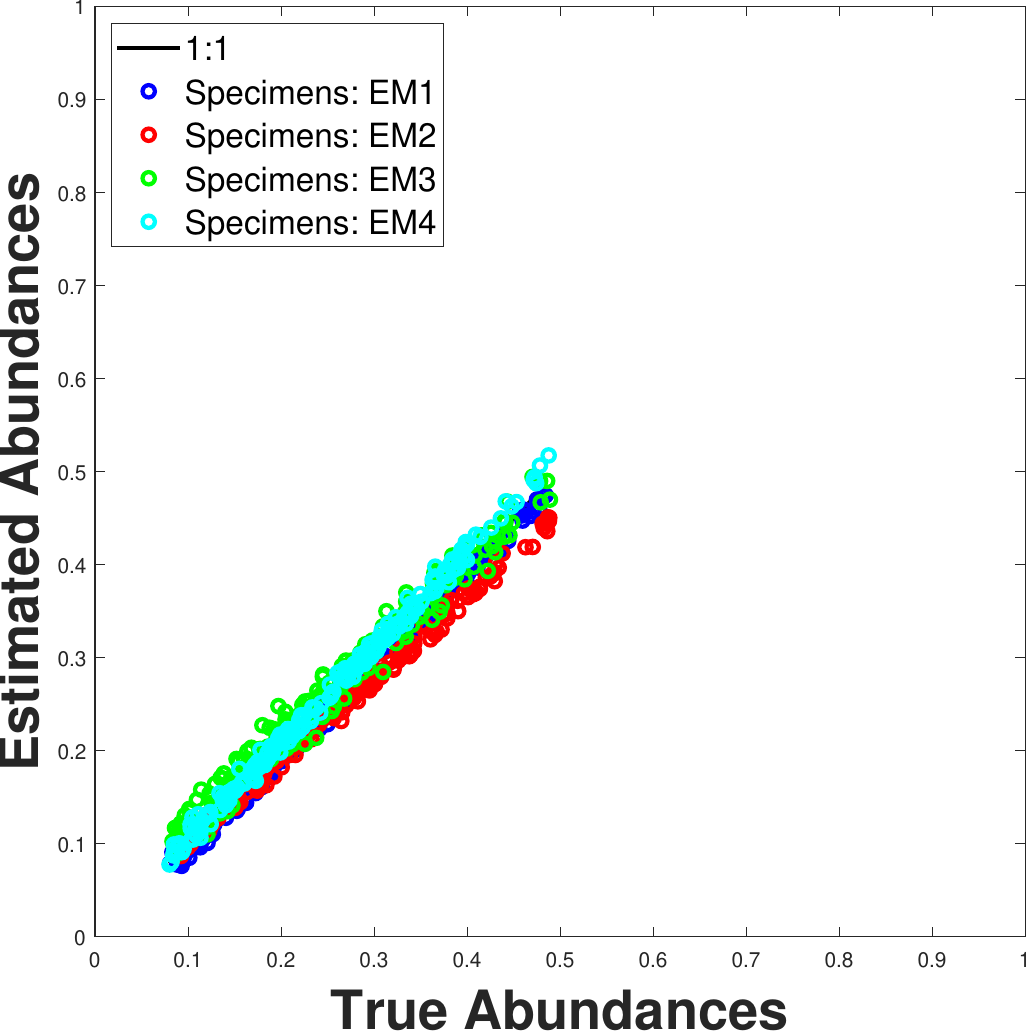}
    \caption{MAD-NMF: Abundances}
    \label{F: noisycoversanddatamaxabundances}
 \end{subfigure}
 \caption{Noisy highly mixed coversand dataset: $\alpha = 100$, $\lambda = 2.6$ (a and b) AnalySize; (c and d) MAD-NMF.}\label{F: noisycoversanddata}
    \end{figure}

Table~\ref{T: coversandangle} reports the MAEM and MAAB. MAD-NMF achieves lower values than AnalySize, demonstrating its superior ability to recover the end members and abundances, which is consistent with the visual results shown above. Table~\ref{T: coversanddistanceterm} reports the distances between end-members and their centroid. The results show that MAD-NMF produces larger distance than AnalySize.

\begin{table}[H]
\caption{MAEM and MAAB for (noisy) highly mixed coversand dataset}
    \label{T: coversandangle}
    \centering
    \begin{tabular}{rrr}
    \hline 
    Methods & AnalySize & MAD-NMF\\
    \hline
Highly mixed coversand dataset &&\\
 MAEM &  15.6167  &  1.5787
\\
 MAAB &   14.9915   & 1.4541
\\
 \hline
Noisy highly mixed coversand dataset &&\\
 MAEM &  15.5002  &  2.0862  \\
 MAAB &14.9172  &  3.1873\\
    \hline
    \end{tabular}
\end{table}

\begin{table}[H]
\caption{The $||\bm{C}_{K\times K}\bm{S}||_F^2$ for (noisy) highly mixed coversand dataset}
    \label{T: coversanddistanceterm}
    \centering
    \begin{tabular}{rrr}
    \hline 
    Methods & AnalySize & MAD-NMF\\
    \hline
   Highly mixed coversand dataset & 0.0465  &  0.1801
 \\
  Noisy highly mixed coversand dataset &     0.0474  &  0.1907
 \\
    \hline
    \end{tabular}
\end{table}

\section{Conclusion}\label{s: con}

In this paper, we propose maximum-distance nonnegative matrix factorization (MAD-NMF) for highly mixed grain-size distribution datasets. The proposed method is optimized using a hierarchical alternating least squares (HALS) algorithm. MAD-NMF includes a negative distance regularization term that encourages greater separation among the end members, allowing them to be located farther from the observed specimens. This property makes MAD-NMF suitable for highly mixed datasets.

Following the experimental settings of \citet{van2018genetically, paterson2015new, ZHANG2020106656, qi2026unmixing}, we evaluated MAD-NMF on (noisy) highly mixed two-end-member dataset and (noisy) highly mixed coversand datasets. The results demonstrate that AnalySize fails to recover the true end members and abundances from highly mixed datasets, whereas MAD-NMF effectively recovers the true end members and abundances.

\section*{Funding}

Zhongming Chen is partially supported by Natural Science Foundation of China (No. 12671627) and Natural Science Foundation of Xinjiang Uygur Autonomous Region (No. 2024D01A09).

\section*{Conflict of interest}

Not applicable.

\section*{Availability of data and software}

The four coversand end members are from \citet{ZHANG2020106656code}. Code is available on \citet{qi2026madnmfcode}.

\bibliography{references}

\end{document}